%% file: main.tex
\documentclass{article}
\usepackage{ijcai20}
\input{macros}

\usepackage{etoolbox}
\newtoggle{longversion}
\togglefalse{longversion}
\iftoggle{longversion}{}

\title{An Online Learning Framework for\\ Energy-Efficient Navigation of Electric Vehicles}

\author{
Niklas Åkerblom$^{1,3}$\footnote{Contact Author}
\and
Yuxin Chen$^2$
\And
Morteza Haghir Chehreghani$^3$\\
\affiliations
$^1$Volvo Car Corporation\\
$^2$The University of Chicago\\
$^3$Chalmers University of Technology\\
\emails
niklas.akerblom@chalmers.se, chenyuxin@uchicago.edu, morteza.chehreghani@chalmers.se
}

\begin{document}

\maketitle

\input{body}

\clearpage
\bibliographystyle{named}
\bibliography{main}

\end{document}

%% file: macros.tex
\usepackage{times}

\usepackage{soul}
\usepackage{url}
\usepackage[hidelinks]{hyperref}
\usepackage[utf8]{inputenc}
\usepackage[small]{caption}
\usepackage{multirow,color,graphicx}

\usepackage{amsmath,amsthm}
\usepackage{booktabs}
\usepackage{algorithm}
\usepackage[noend]{algpseudocode}

\usepackage{color,soul}
\usepackage[dvipsnames]{xcolor}
\usepackage{hyperref}
\hypersetup{
    colorlinks = true,
    linkcolor = black,
    urlcolor  = NavyBlue,
    citecolor = black,
}

\graphicspath{ {} }
\usepackage{subcaption}

\newtheorem{theorem}{Theorem}

\theoremstyle{definition}

\theoremstyle{remark}

\numberwithin{equation}{section}

\newcommand{\paren} [1] {\ensuremath{ \left( {#1} \right) }}

\newcommand{\condparen}[2]{\ensuremath{\left(#1\left\lvert#2\right.\right)}}

\newcommand{\bigO}[1]{\ensuremath{O\paren{#1}}}

\newcommand{\bigOTilde}[1]{\ensuremath{\tilde{O}\paren{#1}}}

\newcommand{\expct}[1]{\mathbb{E}\left[#1\right]}

\usepackage{amssymb}

\DeclareMathOperator{\lognormal}{LogNormal}

%% file: body.tex
\begin{abstract}
Energy-efficient navigation constitutes an important challenge in electric vehicles, due to their limited battery capacity. We employ a Bayesian approach to model the energy consumption at road segments for efficient navigation. In order to learn the model parameters, we develop an online learning framework and investigate several exploration strategies such as Thompson Sampling and Upper Confidence Bound. We then extend our online learning framework to multi-agent setting, where multiple vehicles adaptively navigate and learn the parameters of the energy model. We analyze Thompson Sampling and establish rigorous regret bounds on its performance. Finally, we demonstrate the performance of our methods via several real-world experiments on Luxembourg SUMO Traffic dataset.
\end{abstract}

\section{Introduction}
Today, electric vehicles experience a fast-growing role in transport sys.pdftems. However, the applicability of these systems are often confined with the limited capacity of their batteries.
A common concern of electric vehicles is the so-called ``range anxiety'' issue. Due to the historically high cost of batteries, the range of electric vehicles has generally been much shorter than that of conventional vehicles, which has led to the fear of being stranded when the battery is depleted. Such concerns could be alleviated by improving the navigation algorithms and route planning methods for these systems. Therefore, in this paper we aim at developing principled methods for energy-efficient navigation of electric vehicles. 

Several works employ variants of shortest path algorithms for the purpose of finding the routes that minimize the energy consumption. Some of them, e.g. \cite{artmeier2010shortest,sachenbacher2011efficient}, focus on computational efficiency in searching for feasible paths where the constraints induced by limited battery capacity are satisfied. Both works use energy consumption as edge weights for the shortest path problem. They also consider recuperation of energy modeled as negative edge weights, since they identify that negative cycles can not occur due to the law of conservation of energy. In \cite{sachenbacher2011efficient} a consistent heuristic function for energy consumption is used with a modified version of A*-search to capture battery constraints at query-time.
In \cite{baum_et_al:LIPIcs:2017:7608}, instead of using fixed scalar energy consumption edge weights, the authors use piecewise linear functions to represent the energy demand, as well as lower and upper limits on battery capacity. 
This task has also been developed beyond the shortest path problems in the context of the well-known vehicle routing problem (VRP). In  \cite{basso2019energy}, VRP is applied to electrified commercial vehicles in a two-stage approach, where the first stage consists of finding the paths between customers with the lowest energy consumption and at the second stage the VRP including optional public charging station nodes is solved.

The aforementioned methods either assume the necessary information for computing the optimal path is available, or do not provide any satisfactory exploration to acquire it. Thereby, we focus on developing an \emph{online} framework to learn (explore) the parameters of the energy model adaptively alongside solving the navigation (optimization) problem instances. 
We will employ a Bayesian approach to model the energy consumption for each road segment. The goal is to learn the parameters of such an energy model to be used for efficient navigation. Therefore, we will develop an online learning framework to investigate and analyze several exploration strategies for learning the unknown parameters. 

Thompson Sampling (TS) \cite{thompson1933}, also called \emph{posterior sampling} and \emph{probability matching}, is a model-based exploration method for an optimal trade-off between exploration and exploitation. 
Several experimental \cite{ChapelleL11,GraepelCBH10,ChenRCK17} and theoretical studies \cite{osband2017posterior,BubeckL13,KaufmannKM12} have shown the effectiveness of Thompson Sampling in different settings. \cite{ChenRCK17} develop an online framework to explore the parameters of a decision model via Thompson Sampling in the application of interactive troubleshooting. 
\cite{wang2018thompson} use Thompson Sampling for combinatorial semi-bandits including the shortest path problem with Bernoulli distributed edge costs, and derives distribution-dependent regret bounds.

Upper Confidence Bound (UCB) \cite{Auer02} is another approach used widely for exploration-exploitation trade-off. A variant of UCB for combinatorial semi-bandits is introduced and analyzed in \cite{chen2013combinatorial}.
A Bayesian version of the Upper Confidence Bound method is introduced in \cite{kaufmann2012bayesian} and later analyzed in terms of regret bounds in \cite{kaufmann2018bayesian}. An alternative approach is proposed in \cite{reverdy2014modeling}. 

Beyond the novel online learning framework for energy-efficient navigation, we further extend our algorithms to the multi-agent setting, where multiple vehicles adaptively navigate and learn the parameters of the energy model. We then extensively evaluate the proposed algorithms on several synthetic navigation tasks, as well as on a real-world setting on Luxembourg SUMO Traffic dataset. 

\section{Energy Model}
We model the road network by a directed graph $\mathcal{G}(V, E, w)$ where each vertex $v \in V$ represents an intersection of the road segments, and $E$ indicates the set of directed edges. Each edge $e = (u, v) \in E$ is an ordered pair of vertices $u, v \in V$ such that $u \neq v$ and it represents the road segment between the intersections associated with $u$ and $v$. In the cases where bidirectional travel is allowed on a road segment represented by $(u,v) \in E$, we add an edge $(v,u) \in E$ in the opposite direction. A directed \emph{path} is a sequence of vertices $\langle v_1, v_2, \dots, v_n \rangle$, where $v_i \in V$ for $i = 1, \dots, n$ and $(v_i, v_{i+1}) \in E$ for $i = 1, \dots, n-1$. Hence, a path $p$ can also be viewed as a sequence of edges. If $p$ starts and ends with the same vertex, $p$ is called a cycle.

\looseness -1 We associate a weight function $w : E \rightarrow \mathbb{R}^+$ to each edge in the graph, representing the total energy consumed by a vehicle traversing that edge. The weight function is extended to a path $p$ by letting $w(p) = \sum_{e \in p} w(e)$. We may define other functions to specify the other attributes associated with intersections and road segments, such as the average speed $b : E \rightarrow \mathbb{R}^+$, the length $d : E \rightarrow \mathbb{R}^+$, and the inclination $\theta : E \rightarrow \mathbb{R}$.

\looseness -1 In our setting, the energy consumptions at different road segments are stochastic and a priori unknown. 
We adopt a Bayesian approach to model the energy consumption at each road segment $e\in E$, i.e., the edge weights. Such a choice provides a principled way to induce prior knowledge. Furthermore, as we will see, this approach fits  well with the online learning and exploration of the parameters of the energy model. 

We first consider a deterministic model of vehicle energy consumption $\epsilon(e)$ for an edge $e$, which will be used later as the prior. Similar to the model in \cite{basso2019energy}, our model is based on longitudinal vehicle dynamics. For convenience, we assume that vehicles drive with constant speed along individual edges so that we can disregard the longitudinal acceleration term. However, this assumption is only used for the prior. We then have the following equation for the approximated energy consumption (in watt-hours)
\begin{eqnarray}
    \label{eq:energy_consumption}
    \epsilon(e) &=& \frac{m g d(e) \sin (\theta(e)) + m g C_r d(e) \cos (\theta(e))}{3600 \eta} \nonumber\\
    && + \frac{0.5 C_d A \rho d(e) b^2(e)}{3600 \eta}
\end{eqnarray}

In Equation \ref{eq:energy_consumption} the vehicle mass $m$, the rolling resistance coefficient $C_r$, the front surface area $A$ and the air drag coefficient $C_d$ are vehicle-specific parameters. Whereas, the road segment length $d$, speed $b$ and inclination angle $\theta$ are location (edge) dependent. We treat the gravitational acceleration $g$ and air density $\rho$ as constants. The powertrain efficiency $\eta$ is vehicle specific and can be approximated as a quadratic function of the speed $b$ or by a constant $\eta = 1$ for an ideal vehicle with no battery-to-wheel energy losses.

Actual energy consumption can be either positive (traction and auxiliary loads like air conditioning) or negative (regenerative braking). If the energy consumption is modeled accurately and used as $w(e)$ in a graph $\mathcal{G}(V,E,w)$, the law of conservation of energy guarantees that there exists no cycle $c$ in $\mathcal{G}$ where $w(c) < 0$. However, since we are estimating the expected energy consumption from observations, this guarantee does not necessarily hold in our case. 

Thereby, while modeling energy recuperation is desirable from an accuracy perspective, it introduces some difficulties. In terms of computational complexity, Dijkstra's algorithm does not allow negative edge weights and Bellman-Ford's algorithm is slower by an order of magnitude. While there are methods to overcome this, they still assume that there are no negative edge-weight cycles in the network. Hence, we choose to only consider positive edge-weights when solving the energy-efficient (shortest path) problem. This approximation should still achieve meaningful results, since even with discarding recuperation, edges with high energy consumption will still be avoided.
So while the efficiency function $\eta$ has a higher value when the energy consumption is negative than when it is positive, we believe using a constant is a justified simplification as we only consider positive edge-level energy consumption in the optimization stage.\footnote{We emphasize that our generic online learning framework is independent of such approximations, and can be employed with any senseful energy model.}

Motivated by \cite{wu2015electric}, as the first attempt,  we assume the observed energy consumption $y(e)$ of a road segment represented by an edge $e$ follows a Gaussian distribution, given a certain small range of inclination, vehicle speed and acceleration. We also assume that $y(e)$ is independent from $y(e')$ for all $e' \in E$ where $e' \neq e$ and that we may observe negative energy consumption. The likelihood is then
\begin{equation*}
    \label{eq:gaussian_likelihood}
    p(y(e) \mid \mu(e), \sigma_s^2(e)) = \mathcal{N}(y(e) \mid \mu(e), \sigma_s^2(e))
\end{equation*}

Here, for clarity we assume the noise variance $\sigma_s^2$ is given.
We can then use a Gaussian conjugate prior over the mean energy consumption:
\begin{equation*}
    \label{eq:gaussian_prior}
    p(\mu(e) \mid \mu_0(e), \sigma_0^2(e)) = \mathcal{N}(\mu(e) \mid \mu_0(e), \sigma_0^2(e)) , 
\end{equation*}
where we choose $\mu_0(e) = \epsilon(e)$ and $\sigma_0^2(e) = (\vartheta \mu_0(e))^2$ for some constant $\vartheta > 0$. Due to the conjugacy properties, we have closed-form expressions for updating the posterior distributions with new observations of $y(e)$. For any path $p$ in $\mathcal{G}$, we have $\mathbb{E}[\sum_{e \in p} y(e)] = \sum_{e \in p} \mathbb{E}[y(e)]$, which means we can find the path with the lowest expected energy demand if we set $w(e) = \mathbb{E}[y(e)]$ and solve the shortest path problem over $\mathcal{G}(V,E,w)$. 
To deal with $\mathbb{E}[y(e)] < 0$, we instead set $w(e) = \mathbb{E}[z(e)]$ where $z(e)$ is distributed according to the rectified normal distribution $\mathcal{N}^R(\mu(e), \sigma_s^2(e))$, which is defined so that $z(e) = \max{(0, y(e))}$ and $y(e) \sim \mathcal{N}(\mu(e), \sigma_s^2(e))$. The expected value is then calculated as $\mathbb{E}[z(e)] = \mu(e)(1 - \Phi(-\mu(e)/\sigma(e))) + \sigma(e) \phi(-\mu(e)/\sigma(e))$, where $\Phi(x)$ and $\phi(x)$ are the standard Gaussian CDF and PDF respectively.

Alternatively, instead of assuming a rectified Gaussian distribution for the energy consumption of each edge, we model the non-negative edge weights by (conjugate) Log-Gaussian likelihood and prior distributions. By definition, if we have a Log-Gaussian random variable $Z \sim \lognormal(\mu, \sigma^2)$, then the logarithm of $Z$ is a Gaussian random variable $\ln{Z} \sim \mathcal{N}(\mu, \sigma^2)$. Therefore, we have the expected value $\mathbb{E}[Z] = \exp\{\mu + 0.5 \sigma^2\}$, the variance $\mathbf{Var}[Z] = (\exp\{\sigma^2\} - 1)\exp\{2\mu + \sigma^2\}$ and the mode $\mathbf{Mode}[Z] = \exp\{\mu - \sigma^2\}$. We can thus define the likelihood as
\begin{align}
    \label{eq:loggaussian_likelihood}
    &p\condparen{y(e)}{\mu(e), \sigma_s^2(e)} \nonumber\\ &\quad\;=\lognormal\condparen{y(e)}{\ln{\mu(e)} - \frac{\sigma_s^2(e)}{2}, \sigma_s^2(e)}
\end{align}
where $\mathbb{E}[y(e)] = \mu(e)$ and $\mathbf{Var}[y(e)] = (\exp\{\sigma_s^2(e)\} - 1)\mathbb{E}[y(e)]^2 $. We also choose the prior hyperparameters such that $\mathbb{E}[\mu(e)] = \mu_0(e)$ and $\mathbf{Var}[\mu(e)] = \sigma_0^2(e)$, where $\mu_0(e)$ and $\sigma_0^2(e)$ are calculated in the same way as for the Gaussian prior, in order to make fair comparisons between the Gaussian and Log-Gaussian results. The resulting prior distribution is
\begin{align}
    \label{eq:loggaussian_prior}
    p&\big(\mu(e) \big\vert \mu_0(e), \sigma_0^2(e)\big) =  \lognormal\left(\mu(e) \big\vert \right. \nonumber\\
    &\left. \ln{\mu_0(e)} - \frac{1}{2} \ln\paren{1 + \frac{\sigma_0^2(e)}{\mu_0^2(e)}},~\ln\paren{1 + \frac{\sigma_0^2(e)}{\mu_0^2(e)}} \right)
\end{align}

\section{Online Learning and Exploration of the Energy Model}\label{sec:online-single}
We develop an \emph{online learning} framework to explore the parameters of the energy model adaptively alongside solving sequentially the navigation (optimization) problem at different sessions.
At the beginning, the exact energy consumption of the road segments and the parameters of the respective model are unknown. Thus, we start with an approximate and possibly inaccurate estimate of the parameters. We use the current estimates to solve the current navigation task. We then update the model parameters according to the observed energy consumption at different segments (edges) of the navigated path, and use the new parameters to solve the next problem instance. 

Algorithm \ref{alg:online_algorithm} describes these steps, where $\mu_t$ and $\sigma_t^2$ refer to the current parameters of the energy model for all the edges at the current session $t$, which are used to obtain the current edge weights $w_t$'s. We solve the optimization problem using $w_t$'s to determine the optimal action (or the arm in the nomenclature of multi-armed bandit problems) $a_t$, which in this context is a path. The action $a_t$ is applied and a reward $r(a_t)$ is observed, consisting of the actual measured energy consumption for each of the passed edges. Since we want to minimize energy consumption, we regard it as a negative reward when we update the parameters (shown for example for the Gaussian model in Algorithm \ref{alg:gaussian_update_parameters}). $T$ indicates the total number of sessions, sometimes called the horizon. 
For measuring the effectiveness of our online learning algorithm, we consider its regret, which is the difference in the total expected reward between always playing the optimal action and playing actions according to the algorithm. Formally, 
the instant regret at session $t$ is defined as $\Delta_t := (\max_a \expct{r(a)}) - \expct{r(a_t)}$ where $\max_a \expct{ r(a)}$ is the maximal expected reward for any action, and the cumulative regret is defined as $R_T=\sum_{t=1}^T \Delta_t$.

\begin{algorithm}[tb]
\caption{Online learning for energy-efficient navigation}
\label{alg:online_algorithm}
\begin{algorithmic}[1]
\Require $\mu_0, \sigma^2_0$
\For{$t \leftarrow 0, 1, \dots, T$}
    \label{row:optimization_objective}\State $w_t \leftarrow $ \Call{GetEdgeWeights}{$t, \mu_t, \sigma_t^2$}
    \label{row:optimization_solution}\State $a_t \leftarrow $ \Call{SolveOptimizationToFindAction}{$w_t$}
    \label{row:apply_action}\State $r_t \leftarrow $ \Call{ApplyActionAndObserveReward}{$a_t$}
    \label{row:update_parameters}\State $\mu_{t+1}, \sigma^2_{t+1} \leftarrow $ \Call{UpdateParameters}{$a_t, r_t, \mu_t, \sigma^2_t$}
\EndFor
\end{algorithmic}
\end{algorithm}

\begin{algorithm}[tb]
\caption{Gaussian parameter update of the energy model}
\label{alg:gaussian_update_parameters}
\begin{algorithmic}[1]
\Procedure{UpdateParameters}{$a_t, r_t, \mu_t, \sigma_t^2$}
\For{each edge $e \in a_t$}
\State $\epsilon_t(e) \leftarrow - r_t(e)$
\State $\sigma_{t+1}^2(e) \leftarrow \left(\frac{1}{\sigma_{t}^2(e)} + \frac{1}{\sigma_{s}^2(e)}\right)^{-1}$
\State $\mu_{t+1}(e) \leftarrow \sigma_{t+1}^2(e) \left( \frac{\mu_{t}(e)}{\sigma_{t}^2(e)} + \frac{\epsilon_t(e)}{\sigma_{s}^2(e)}\right)$
\EndFor
\State \Return $\mu_{t+1}, \sigma_{t+1}^2$
\EndProcedure
\end{algorithmic}
\end{algorithm}

\subsection{Shortest Path Problem as Multi-Armed Bandit}

A combinatorial bandit \cite{gai2012combinatorial} is a multi-armed bandit problem where an agent is only allowed to pull sets of arms instead of an individual arm. However, there may be restrictions on the feasible combinations of the arms. We consider the combinatorial semi-bandit case where the rewards are observed for each individual arm pulled by an agent during a round. 

A number of different combinatorial problems can cast to multi-armed bandits in this way, among them the shortest path problem is the focus of this work. Efficient algorithms for the deterministic problem (e.g. Dijkstra's algorithm \cite{dijkstra1959note}) can be used as an oracle \cite{wang2018thompson} to provide feasible sets of arms to the agent, as well as to maximize the expected reward.

We connect this to the optimization problem in Algorithm \ref{alg:online_algorithm}, where we want to find an action $a_t$. At time $t$, let $\mathcal{G}(V,E,w_t)$ be a directed graph with weight function $w_t$ and sets of vertices $V$ and edges $E$. Given a source vertex $u \in V$ and a target vertex $v \in V$, let $P$ be the set of all paths $p$ in $\mathcal{G}$ such that $p = \langle u, \dots, v \rangle$. Assuming non-negative edge costs $w_t(e)$ for each edge $e \in E$, the problem of finding the shortest path (action $a_t$) from $u$ to $v$ can be defined as
\begin{equation}
    a_t = \text{arg}\,\min\limits_{p \in P}\,\sum_{e\in p}w_t(e)
\end{equation}

\subsection{Thompson Sampling}\label{sec:alg:ts}

Since the greedy method does not actively explore the environment, there are other methods which performs better in terms of minimizing cumulative regret. One commonly used method is $\epsilon$-greedy, where a (uniform) random action is taken with probability $\epsilon$ and the greedy strategy is used otherwise. However, this method is not well suited to the shortest path problem, since a random path from the source vertex to the target would almost certainly be very inefficient in terms of accumulated edge costs.

An alternative method for exploration is Thompson Sampling (TS). In our Bayesian setup, the greedy strategy chooses the action which maximizes the expected reward according to the current estimate of the mean rewards. In contrast, with TS the agent samples from the model, i.e., it selects an action which has a high probability of being optimal by sampling mean rewards from the posterior distribution and choosing an action which maximizes those during each session.

Thompson Sampling for the energy consumption shortest path problem is outlined in Algorithm \ref{alg:ts_optimization_objective}, where it can be used in Algorithm \ref{alg:online_algorithm} to obtain the edge weights in the network (only shown for the Gaussian model). In the following, we provide an upper bound on the cumulative regret of Thompson Sampling for the shortest path navigation problem.
\iftoggle{longversion}{\footnote{We defer the full proof to the appendix.}\label{footnote:longer}
}

\begin{theorem}\label{thm:graph_regret_bound_ts}
Let $\mathcal{G}(V, E, w)$ be a weighted directed graph. 
Let $N$ be the number of paths in $\mathcal{G}$. The expected cumulative regret of Algorithm~\ref{alg:ts_optimization_objective} 
satisfies that $\text{BayesRegret}(T) = \bigOTilde{|V|^2\sqrt{|E|T}}$. Furthermore,
$\expct{R_T}=\min\{\bigO{N \log T}, \bigO{|E||V|\log T}\}$.
\end{theorem}
\begin{proof}[Proof sketch]  
The online shortest path problem could be viewed as (1) a combinatorial semi-bandit problem with linear reward functions, where the feedback includes all the sub-arms (edges) in the played arm (path from source to target); or (2) a reinforcement learning problem, where node $v\in V$ corresponds to a state, and each edge $e\in E$ corresponds to an action. For (1), we use (instance dependent) expected cumulative regret bounds $\bigO{N \log T}$ \cite{agrawal2012analysis}  and $\bigO{|E||V|\log T}$ \cite{wang2018thompson}, yielding the (instance dependent) bound 
$\min\{\bigO{N \log T}, \bigO{|E||V|\log T}\}$; for (2), we get a Bayesian regret bound as $\bigOTilde{|V|^2\sqrt{|E|T}}$ \cite{osband2017posterior} (we ignore logarithmic factors). Therefore, combining the two views, 
we obtain a bound on the regret.   
\end{proof}

\begin{algorithm}[tb]
\caption{\label{alg:ts_optimization_objective}Thompson Sampling}
\begin{algorithmic}[1]
\Procedure{GetEdgeWeights}{$t, \mu_t, \sigma_t^2$}
\For{each edge $e \in E$}
\State $\hat{\mu} \leftarrow$ Sample from posterior
$\mathcal{N}(\mu_{t}(e), \sigma^2_{t}(e))$
\State $w_{t}(e) \leftarrow \mathbb{E}[y]$ where $y \sim \mathcal{N}^R(\hat{\mu}, \sigma_s^2(e))$ 
\EndFor
\State \Return $w_t$
\EndProcedure
\end{algorithmic}
\end{algorithm}

\subsection{Upper Confidence Bound}
Another class of algorithms demonstrated to work well in the context of multi-armed bandits is the collection of the methods developed around Upper Confidence Bound (UCB). Informally, these methods are designed based on the principle of optimism in the face of uncertainty. The algorithms achieve efficient exploration by choosing the arm with the highest empirical mean reward added to an exploration term (the confidence width). Hence, the arms chosen are those with a plausible possibility of being optimal.

In \cite{chen2013combinatorial} a combinatorial version of UCB (CUCB) is shown to achieve sublinear regret for combinatorial semi-bandits. However, using a Bayesian approach is beneficial in this problem since it allows us to employ the  theoretical knowledge on the energy consumption in a prior. Hence, we consider BayesUCB \cite{kaufmann2012bayesian} and adapt it to the combinatorial semi-bandit setting. Similar to \cite{kaufmann2012bayesian}, we denote the quantile function for a distribution $\lambda$ as $Q(\alpha,\lambda)$ such that $\mathbb{P}_\lambda (X \leq Q(\alpha, \lambda)) = \alpha $. The idea of that work is to use upper quantiles of the posterior distributions of the expected arm rewards to select arms. If $\lambda$ denotes the posterior distribution of an arm and $t$ is the current session, the Bayesian Upper Confidence Bound (BayesUCB) is $Q(1 - 1/t, \lambda)$. 

This method is outlined in Algorithm \ref{alg:ucb_algorithm} for the Gaussian model. Here, since the goal is to minimize the energy consumption which can be considered as the negative of  the reward, thus, we use the lower quantile $Q(1/t, \lambda)$.

\begin{algorithm}[tb]
\caption{\label{alg:ucb_algorithm}BayesUCB}
\begin{algorithmic}[1]
\Procedure{GetEdgeWeights}{$t, \mu_t, \sigma_t^2$}
\For{each edge $e \in E$}
\State $w_{t}(e) \leftarrow \max \left(0, Q\left(\frac{1}{t}, \mathcal{N}(\mu_{t}(e),\sigma_{t}^2(e))\right)\right)$
\EndFor
\State \Return $w_t$
\EndProcedure
\end{algorithmic}
\end{algorithm}

\begin{figure*}
  \centering
  \begin{subfigure}[b]{.385\textwidth}
    \centering
    {
      \includegraphics[trim={5pt 15pt 0pt 20pt}, width=\textwidth]{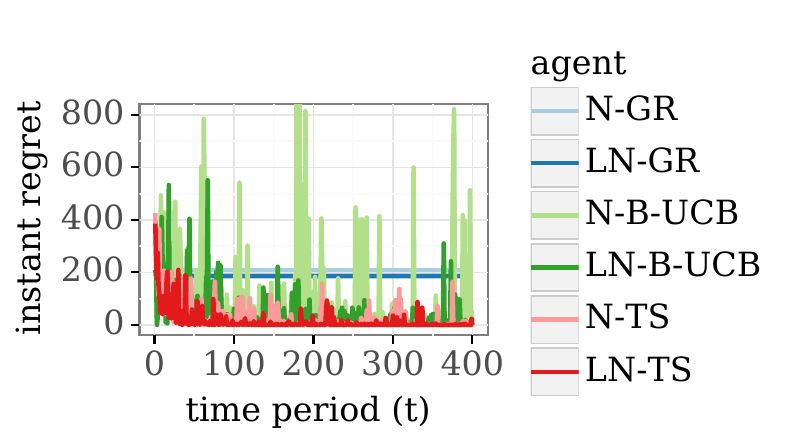}
      \caption{}
      \label{fig:instant_regret}
    }
  \end{subfigure}
  \hspace{-2mm}
  \begin{subfigure}[b]{.395\textwidth}
    \centering
    {
      \includegraphics[trim={0pt 15pt 0pt 20pt}, width=\textwidth]{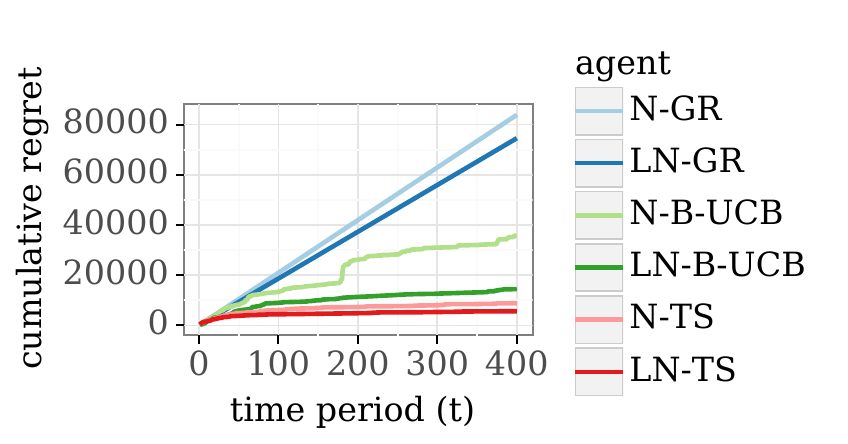}
      \caption{}
      \label{fig:cumulative_regret}
    }
  \end{subfigure}
  \hspace{-3mm}
  \begin{subfigure}[b]{.21\textwidth}
    \centering
    {
      \includegraphics[trim={0pt 15pt 5pt 20pt}, width=\textwidth]{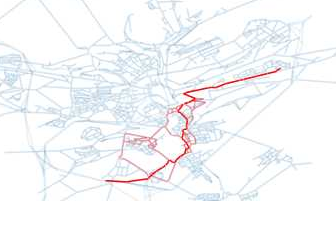}
      \caption{}
      \label{fig:exploration_ts}
    }
  \end{subfigure}
    \caption{Experimental results on the real-world dataset. (a) Instant regret for Thompson Sampling (TS), BayesUCB (B-UCB) and probabilistic greedy (GR) algorithms, applied to Gaussian (prefix N) and Log-Gaussian bandits (prefix LN). (b) Cumulative regret results. (c) Exploration with Thompson Sampling, where the red lines indicate the edges visited by the agent during exploration.}
    \label{fig:exp:results_real}
\end{figure*}

\section{Multi-Agent Learning and Exploration}
\label{sec:multi_agent}

The online learning may speed up via having multiple agents exploring simultaneously and sharing information on the observed rewards with each other. In our particular application, this corresponds to a fleet of vehicles of similar type sharing information about energy consumption across the fleet. Such a setting can be very important for road planning, electric vehicle industries and city principals. 

The communication between the agents for the sake of sharing the the observed rewards can be synchronous or asynchronous. In this paper, we consider the synchronous setting, where the vehicles drive concurrently in each time step and share their accumulated knowledge with the fleet before the next iteration starts. At each session, any individual vehicle independently selects a path to explore/exploit according to the online learning strategies provided in Section \ref{sec:online-single}. It is notable that our online learning framework and theoretical analysis are applicable to the asynchronous setting in a similar manner. 
Below, we provide a regret bound for the TS-based multi-agent learning algorithm under the synchronous setting.
\begin{theorem}[Synchronous Multi-agent Learning]\label{thm:ma_regret_bound_ts} Let $K$ be the number of agents,  
and $T$ be the number of sessions. Given a weighted directed graph $\mathcal{G}(V, E, w)$, the expected cumulative regret of the synchronized multi-agent online learning algorithm (i.e., $K$ agents working in parallel in each session) invoking Algorithm \ref{alg:ts_optimization_objective} satisfies
$\expct{R^F_T} = O\left(NK + \expct{R_{TK}}\right)$, where $R^F_T$ is the total (horizon $T$) multi-agent regret, $R_{TK}$ is the single-agent (horizon $TK$) regret of Algorithm \ref{alg:ts_optimization_objective} and $N$ is the number of paths in $\mathcal{G}$.
\end{theorem}

\begin{proof}[Proof sketch]
The proof considers the online shortest path problem as a combinatorial semi-bandit problem, and treats the multi-agent setting as a sequential algorithm with delayed feedback. The result is obtained as a corollary of 
Theorem \ref{thm:graph_regret_bound_ts} of Section \ref{sec:alg:ts} and Theorem 6 of \cite{joulani2013online} which converts online algorithms for the nondelayed case to ones that can handle delays in the feedback, while retaining their theoretical guarantees.
\end{proof}


\begin{figure*}[!t]
  \centering
  \begin{subfigure}[b]{.35\textwidth}
    \centering
    {
      \includegraphics[trim={5pt 15pt 0pt 20pt}, width=\textwidth]{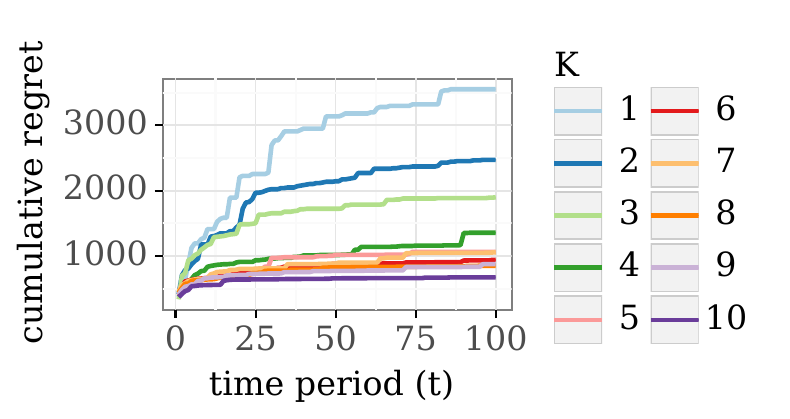}
      \caption{}
      \label{fig:multi_cumulative_regret}
    }
  \end{subfigure}
  \begin{subfigure}[b]{.31\textwidth}
    \centering
    {
      \includegraphics[trim={0pt 15pt 0pt 20pt}, width=\textwidth]{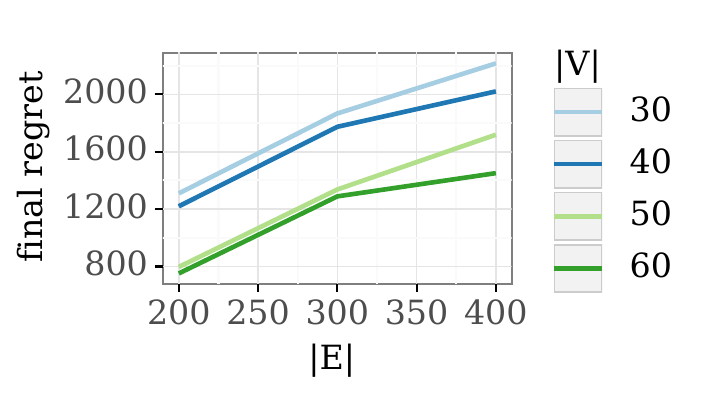}
      \caption{}
      \label{fig:cumulative_regret_edges}
    }
  \end{subfigure}
  \begin{subfigure}[b]{.31\textwidth}
    \centering
    {
      \includegraphics[trim={0pt 15pt 20pt 20pt}, width=\textwidth]{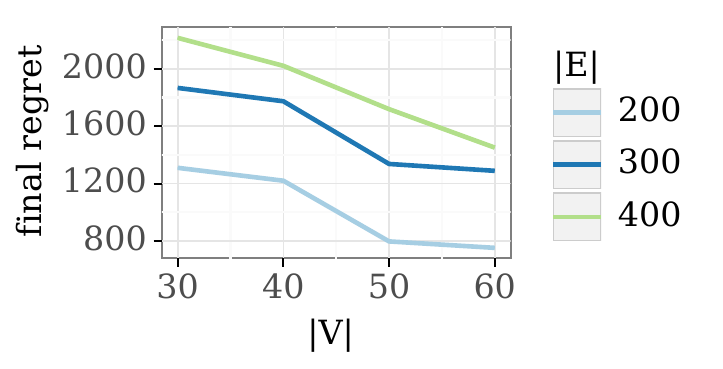}
      \caption{}
      \label{fig:cumulative_regret_vertices}
    }
  \end{subfigure}
    \caption{Experimental results in the multi-agent setting and on the synthetic networks. (a) Average cumulative regret for Thompson Sampling in the multi-agent setting. $K$ denotes the number of agents. (b) Final cumulative regret ($T = 2000$) on synthetic networks as a function of $|E|$. (c) Final cumulative regret ($T = 2000$) on synthetic networks as a function of $|V|$.}
    \label{fig:exp:results_sync}
\end{figure*}


\section{Experimental Results}

In this section, we describe different experimental studies.
For real-world experiments, we extend the simulation framework presented in \cite{tstutorial} to network/graph bandits with general directed graphs, in order to enable exploration scenarios in realistic road networks. 
Furthermore, we add the ability to generate synthetic networks of specified size to this framework, in order to verify the derived regret bounds (as the ground truth is provided for the synthetic networks).

\subsection{Real-World Experiments} \label{sec:exp:rw}


We utilize the Luxembourg SUMO Traffic (LuST) Scenario data \cite{codeca2017luxembourg} to provide realistic traffic patterns and vehicle speed distributions for each hour of the day. This is used in conjunction with altitude and distances from map data, as well as vehicle parameters from an electric vehicle.
The resulting graph $\mathcal{G}$ has $|V| = 2247$ nodes and $|E| = 5651$ edges, representing a road network with $955$ km of highways, arterial roads and residential streets. The difference in altitude between the lowest point in the road network and the highest is $157$ meters. 

We use the default vehicle parameters that were provided for the energy consumption model in \cite{basso2019energy}, with vehicle frontal surface area $A = 8$ meters, air drag coefficient $C_d = 0.7$ and rolling resistance coefficient $C_r = 0.0064$. The vehicle is a medium duty truck with vehicle mass $m = 14750$ kg, which is the curb weight added to half of the payload capacity. 

We approximate the powertrain efficiency during traction by $\eta^{+} = 0.88$ and powertrain efficiency during regeneration by $\eta^{-} = 1.2$. In addition, we use the constant gravitational acceleration $g = 9.81$ $\text{m/s}^2$ and air density $\rho = 1.2$ $\text{kg/m}^3$.

To simulate the ground truth of the energy consumption, we take the average speed $b(e)$ of each edge $e$ from a full 24 hour scenario in the LuST traffic simulation environment. In particular, we observe the values during a peak hour (8 AM), with approximately 5500 vehicles active in the network. This hour is selected to increase the risk of traffic congestion, hence finding the optimal path becomes more challenging. We also get the variance of the speed of each road segment from LuST. Using this information, we sample the speed value for each visited edge  and use the energy consumption model to generate the rewards for the actions.

For the Gaussian likelihood $p(\epsilon(e) | \mu_0, \sigma_s^2)$, we assume $\sigma_s$ to be proportional to $\epsilon(e)$ in Equation \ref{eq:energy_consumption}, such that $\sigma_s^2(e) = (\varphi \epsilon(e))^2$. For the Log-Gaussian likelihood, we choose $\sigma_s^2(e) = \ln(1 - (\varphi \epsilon(e))^2 / (\mu_0^2(e)))$, so that it has the same variance as the Gaussian likelihood. We set $\varphi = 0.1$ for both.
For the prior $p(\mu(e) | \mu_0(e), \sigma_0^2(e))$ of an edge $e \in E$, we use the speed limit of $e$ as $b(e)$, indicating that the average speed is unknown. Then $\mu_0(e) = \epsilon(e)$ and $\sigma_0^2 = (\vartheta \mu_0(e))^2$, where $\vartheta = 0.25$.


As a baseline, we consider the greedy algorithm for both the Gaussian and Log-Gaussian models, where the exploration rule is to always choose the path with the lowest currently estimated expected energy consumption,  an extension of the recent method in \cite{basso2019energy}.

\begin{table}
\centering
\begin{tabular}{lrr}  
\toprule
Agent  & Avg. Regret (Wh) \\
\midrule
N-Greedy & 83851.5 \\
LN-Greedy & 74550.8 \\
N-BayesUCB & 35792.9 \\
LN-BayesUCB & 14452.8 \\
N-TS & 8820.7 \\
LN-TS & 5664.2 \\
\bottomrule
\end{tabular}
\caption{Average cumulative regret at $t=400$.}
\label{tab:cumulative_regret}
\end{table}

We run the simulations with a horizon of $T = 400$ (i.e., $T=400$ sessions). Figure \ref{fig:cumulative_regret} and Table \ref{tab:cumulative_regret} show the cumulative regret for the Gaussian and Log-Gaussian models, where the regret is averaged over 5 runs for each agent. The intuition is that the energy saved by using the TS and UCB agents instead of the baseline greedy agent is the difference in regret, expressed in watt-hours. In Figure \ref{fig:instant_regret}, instant regret averaged over 5 runs is shown for the same scenario. It is clear that Thompson Sampling with the Log-Gaussian model has the best performance in terms of cumulative regret, but the other non-greedy agents also achieve good results. To illustrate  Thompson Sampling explores the road network in a reasonable way, Figure \ref{fig:exploration_ts} visualizes the road network and the paths visited by this exploration algorithm. We observe that no significant detours are performed, in the sense that most paths are close to the optimal path. This indicates the superiority  of Thompson Sampling to a random exploration method such as $\epsilon$-greedy in our application.

For the multi-agent case, we use a horizon of $T = 100$ and 10 scenarios where we vary the number of concurrent agents by $K \in [1, 10]$. The cumulative regret averaged over the agents in each scenario is shown in Figure \ref{fig:multi_cumulative_regret} for each $K$. In the figure, the final cumulative regret for each agent decreases sharply with the addition of just a few agents to the fleet. This continues until there are five agents, after which there seems to be diminishing returns in adding more agents. While there is some overhead (parallelism cost), just enabling two agents to share knowledge with each other decreases their average cumulative regret at $t=T$ by almost a third. This observation highlights the benefit of providing collaboration early in the exploration process, which is also supported by the regret bound in Theorem \ref{thm:ma_regret_bound_ts}. 

\subsection{Synthetic Networks} \label{sec:exp:sn}

In order to evaluate the regret bound in Theorem \ref{thm:graph_regret_bound_ts}, we design synthetic directed acyclic network instances $\mathcal{G}(V,E,w)$ according to a specified number of vertices $n$ and number of edges $q$ (with the constraint that $n - 1 \leq q \leq n (n - 1) / 2$). We start the procedure by adding $n$ vertices $v_1, \dots, v_{n}$ to $V$. Then for each $i \in [1,n-1]$ we add an edge $(i,i+1)$ to $E$. This ensures that the network contains a path with all vertices in $V$. Finally, we add $q-n$ edges $(i,j)$ uniformly at random to $E$, such that $i \neq j$, $i+1 \neq j$ and $i < j$.

Since these networks are synthetic, instead of modeling probabilistic energy consumption, we design instances where it is difficult for an exploration algorithm to find the path with the lowest expected cost. Given a synthetic network $\mathcal{G}$ generated according to the aforementioned procedure, we select $p = \langle v_1, \dots, v_n \rangle$ to be the optimal path. In other words, $p$ contains every vertex $v \in V$. The reward distribution for each edge $e$ in $p$ is chosen to be $\mathcal{N}(\epsilon(e) | \mu(e), \sigma_s^2(e))$ with $\mu(e) = 10$ and $\sigma_s^2(e) = 4$. For $(v_i, v_j) \in E$ where $(v_i, v_j) \notin p$, we set $\mu(e) = 11(j-i)$, where $j-i$ is the number of vertices skipped by the shortcut. This guarantees that no matter the size of the network and the number of edges that form shortcuts between vertices in $p$, $p$ will always have a lower expected cost than any other path in $\mathcal{G}$.

For the prior $\mathcal{N}(\mu(e) | \mu_0(e), \sigma_0^2(e))$, we set $\mu_0(e) = 11(j-i)$ and $\sigma_0^2(e) = 8$. This choice implies according to our prior, every path from the source $v_1$ to the target $v_n$ will initially have the same estimated expected cost.

We run the synthetic network experiment with $T = 2000$ sessions, varying the number of vertices $|V| \in \{30, 40, 50, 60\}$ and edges $|E| \in \{200, 300, 400\}$. In Figure \ref{fig:cumulative_regret_edges}, each plot represents the cumulative regret at $T = 2000$ for a fixed $|V|$, as a function of $|E|$. In Figure \ref{fig:cumulative_regret_vertices}, we instead look at the regret for fixed $|E|$ as a function of $|V|$. We observe that the regret increases with the number of edges and decreases with the number of vertices. This observation is consistent with the theoretical regret bound in Theorem \ref{thm:graph_regret_bound_ts}. By increasing the number of edges $|E|$ while $|V|$ is fixed, the number of paths $N$ and the other terms increase too, which yields a larger regret bound. On the other hand, with a fixed $|E|$ increasing the number of nodes $|V|$ increases the sparsity, i.e., the number of paths $N$ decreases, which in turn yields a lower regret bound.

\section{Conclusion}
We developed a Bayesian online learning framework for the problem of energy-efficient navigation of electric vehicles. Our Bayesian model assume a Gaussian or Log-Gaussian energy model. To learn  the unknown parameters of the model, we adapted exploration methods such as Thompson Sampling and UCB within the online learning framework. We extended the framework to multi-agent setting and established theoretical regret bounds in different settings. Finally, we demonstrated the performance of the framework with several real-world and synthetic experiments. 

\section*{Acknowledgements}
This work is funded by the Strategic Vehicle Research and Innovation Programme (FFI) of Sweden, through the project EENE. We want to thank the authors of \cite{basso2019energy} for providing us with data and their methods.